\documentclass[letterpaper]{article} 
\usepackage{aaai25}  
\usepackage{times}  
\usepackage{helvet}  
\usepackage{courier}  
\usepackage[hyphens]{url}  
\usepackage{graphicx} 
\usepackage{natbib}  
\usepackage{caption} 
\usepackage{algorithm}
\usepackage{algorithmic}
\usepackage{caption}
\usepackage{subcaption}
\usepackage{newfloat}
\usepackage{listings}
\DeclareCaptionStyle{ruled}{labelfont=normalfont,labelsep=colon,strut=off} 
\floatstyle{ruled}
\newfloat{listing}{tb}{lst}{}
\floatname{listing}{Listing}
\usepackage{tcolorbox}

\usepackage{amsmath}
\usepackage{amsfonts}
\usepackage{amsthm}
\usepackage{amssymb}

\theoremstyle{plain}
\newtheorem{theorem}{Theorem}

\newtheorem{prop}{Theorem}
\newtheorem{proposition}[prop]{Proposition}
\newtheorem{lem}{Theorem}
\newtheorem{lemma}[lem]{Lemma}

\theoremstyle{definition}
\newtheorem{defn}{Theorem}
\newtheorem{definition}[defn]{Definition}
\newtheorem{assump}{Theorem}
\newtheorem{assumption}[assump]{Assumption}

\theoremstyle{remark}

\newtheorem*{theorem*}{Theorem}
\newtheorem*{lemma*}{Lemma}
\newtheorem*{definition*}{Definition}
\newtheorem*{corollary*}{Corollary}
\newtheorem*{remark*}{Remark}

\DeclareMathOperator*{\E}{\mathbb{E}}

\newcommand{\s}{\mathcal{S}}

\newcommand{\A}{\mathcal{A}}

\usepackage{url}
\title{Inferring Transition Dynamics from Value Functions}
\author{
    Jacob Adamczyk\\
}
\affiliations{
    \textsuperscript{\rm  }Department of Physics, University of Massachusetts Boston,\\
    The NSF Institute for Artificial Intelligence and Fundamental Interactions
}

\usepackage{bibentry}

\begin{document}

\maketitle

\begin{abstract}
  In reinforcement learning, the value function is typically trained to solve the Bellman equation, which connects the current value to future values. This temporal dependency hints that the value function may contain implicit information about the environment's transition dynamics. By rearranging the Bellman equation, we show that a converged value function encodes a model of the underlying dynamics of the environment. We build on this insight to propose a simple method for inferring dynamics models directly from the value function, potentially mitigating the need for explicit model learning. Furthermore, we explore the challenges of next-state identifiability, discussing conditions under which the inferred dynamics model is well-defined. Our work provides a theoretical foundation for leveraging value functions in dynamics modeling and opens a new avenue for bridging model-free and model-based reinforcement learning.
\end{abstract}

\section{Introduction}
The use of reinforcement learning (RL) for solving sequential decision-making tasks has grown substantially in recent years, demonstrating its potential to discover novel solutions in complex, unknown environments. Among RL approaches, model-based methods—which involve learning or leveraging a model of the environment's dynamics—have often shown superior sample efficiency compared to model-free methods, particularly in tasks with limited interaction budgets. This advantage stems from the ability of model-based approaches to plan over imagined trajectories, using tools like Monte Carlo tree search to evaluate and improve policies. However, these methods often come at the cost of higher computational requirements and rely on the accuracy of the learned dynamics model.

Despite the promise of model-based techniques, learning an accurate model of the environment remains a significant challenge, particularly in domains where simulators are entirely unavailable or expensive to develop. Even with a well-crafted simulator, discrepancies between simulated and real-world dynamics (i.e., the ``sim-to-real'' gap) can degrade the performance of policies at test time. These challenges highlight the need for efficient and reliable ways to learn and use dynamics models, especially in situations where the environment cannot be directly explored.

In this work, we propose a method for inferring a dynamics model from pre-trained value functions. Our approach builds on a new perspective of the Bellman equation, which lies at the core of RL algorithms. Traditionally, the Bellman equation is used to relate the value of a state to the expected value of its successor states. This temporal relationship supplies information about the environment's transition dynamics to the value function. By a simple rearrangement of the Bellman equation, we make this dynamical information accessible, allowing one to recover a model of the environment directly from a previously computed value function.

This insight opens new possibilities for model-based RL. By re-purposing pre-trained value functions to infer dynamics models, we can potentially enhance task performance in settings like multi-task RL, where the inferred model can be used to solve new tasks with changing reward functions. Our method not only leverages existing value functions more effectively, but also provides a step toward bridging the model-free and model-based communities in RL.

\subsection{Motivation}
In many reinforcement learning (RL) workflows, the value function is trained using interactions with the environment for a fixed reward function. Once trained, the value function (or the derived policy) is often saved and used solely for evaluation purposes, with limited opportunities for reuse when new tasks arise. However, in most practical scenarios, tasks share a common environment structure but otherwise differ only in their reward specifications. For example, in robotics, the physical dynamics of the system remain constant, while the task objectives (e.g., picking up specific objects; moving to specific target locations) will vary.

This observation motivates the need for better leveraging pre-trained value functions, which will be particularly useful in settings where the reward functions are handcrafted, known, or even learned. If a dynamics model is accessible, the agent can utilize this model and the known reward function to solve new tasks, either deriving exact solutions or providing strong initializations for further training. Such an approach can significantly reduce the burden of learning from scratch for each new task.

In offline RL, a pre-collected dataset is used to train an RL agent without further access to the environment. Analogously, we propose to use a pre-collected solution, in the form of a value function, to infer the environment's underlying transition dynamics. The inferred model could then be used to finetune policies or adapt to related tasks analogous to online finetuning, but without needing  additional environment interactions.

By proposing a method to recover a dynamics model from a value function, we aim to bridge the gap between model-free and model-based RL. This approach not only effectively re-purposes previously computed solutions in a novel way, but also offers a path toward improving sample efficiency and task adaptability in various RL settings.
\section{Background}
In this section we will introduce the relevant background material for reinforcement learning and required definitions.
\subsection{Reinforcement Learning}
We will consider discrete or continuous state spaces and discrete action spaces\footnote{If the policy expectation over continuous action spaces can be calculated exactly, the derived results are equally valid.}. The RL problem is then modeled by a Markov Decision Process (MDP), which we represent by the tuple $\langle \s,\A,p,r,\gamma \rangle$ with state space $\s$; action space $\A$; potentially stochastic transition function (dynamics) ${p: \s \times \A \to \s}$; bounded, real reward function ${r: \s \times \A \to \mathbb{R}}$; and the discount factor ${\gamma \in [0,1)}$.

The principle objective in RL is to maximize the total discounted reward expected under a policy $\pi$. That is, to find $\pi^*$ that maximizes the following sum of expected rewards:
\begin{equation}\label{eq:pi_star_defn}
    \pi^* = \arg\max_{\pi}
    \E_{\tau \sim{}p,\pi}
    \left[ \sum_{t=0}^{\infty} \gamma^{t} r(s_t,a_t) 
         \right].
\end{equation}

In the present work, we consider value-based RL methods, where the solution to the RL problem is equivalently defined by its optimal action-value function ($Q^*(s,a)$). The aforementioned optimal policy $\pi^*(a|s)$ is derived from $Q^*$ through a greedy maximization over actions. The optimal value function can be obtained by iterating the following recursive Bellman equation until convergence:
\begin{equation}
    Q^*(s,a) = r(s,a) + \gamma \mathbb{E}_{s' \sim{} p(\cdot|s,a)} \max_{a'} \left( Q^*(s',a') \right).
    \label{eq:bellman}
\end{equation}
In the tabular setting, the exact Bellman equation can be applied until convergence. In the function approximator setting, the $Q$ table is replaced by a parameterized function approximator, $Q_\theta$ and the temporal difference (TD) loss is minimized instead. To simplify the discussion, we will neglect the details of precisely how the $Q$ function is derived, and we focus on general value functions alongside their approximations. 

\subsection{Preliminaries}
For the theoretical discussion in later sections, we will need several definitions and lemmas, which we present in this section. Hereon, we refer to a dynamics model as ``identifiable'' if such a model can be derived to recover the mapping $(s,a)\to s'$ for all $s\in\s, a \in \A$.

We begin with the assumption of deterministic dynamics:
\begin{assumption}
The transition dynamics are deterministic. That is, there exists some $f$ such that $s'=f(s,a)$ for all $s,s'\in\s, a \in \A$.\label{assump:det}
\end{assumption}
As the first step in this line of work, we focus on the case where only a single successor state is expected. Of course, the stochastic case will be more challenging, as it requires e.g. a \textit{density model} over state space. 
Additionally, we assume some knowledge of the pre-trained task:
\begin{assumption}
    The reward function $r(s,a)$ and discount factor $\gamma$ used during training are known. \label{assump:rwd}
\end{assumption}
We believe this is not entirely limiting as the reward function and discount factor are often chosen by hand. Moreover, the error analysis presented here can be easily extended to the case with errors in the reward and discount factor.

Often in machine learning literature, the notion of smoothness given by Lipschitz continuity (bounded derivative, loosely) is useful for controlling errors. Instead, we find it useful to consider a \textit{reverse} Lipschitz condition:
 \begin{definition}[Reverse Lipschitz continuity]
If there exists some $L>0$ such that the function $V$ satisfies
\begin{equation*}
    |V(s_1) - V(s_2)| > L |s_1 - s_2|\; \forall s_1,s_2 \in \mathcal{S}, s_1 \neq s_2
\end{equation*}
then we say $V$ is reverse Lipschitz with constant $L$.
\end{definition}
We use the suggestive notation for the function ``$V$'' as it will prove to be a useful property of the state value function for later analysis. The norm used $|\cdot|$ in the previous definition can be any well-defined norm over the space $\s$, but we assume for concreteness the $\ell_1$ norm throughout. 

The idea of a \textit{reverse} Lipschitz function has been discussed by others~\cite{pmlr-v202-kinoshita23a} (``inverse Lipschitzness''\footnote{We use the word ``reverse'' here to better distinguish from the inverse functions that will be considered later.}). In the present work, this notion is useful in determining when a value function leads to an identifiable dynamics model. It will also help to control the error when the value function used is subject to error. Rather than the typical intuition that a \textit{small} Lipschitz constant improves sample complexity, we have the reverse: a \textit{larger} reverse Lipschitz constant (i.e. steeper slopes) helps in estimating the dynamics model with higher accuracy.

A sufficient condition for existence of the constant $L$ is given by the reverse of Theorem 1 in \cite{rachelson2010locality}. That is, by reversing all inequalities in their proof, one immediately obtains the following:
\begin{lemma}[\citet{rachelson2010locality}]
Consider an MDP with \textbf{reverse} $(L_r, L_p)$-Lipschitz rewards and dynamics, $L_\pi$-Lipschitz continuous policy. Then, the corresponding value function $Q^\pi(s,a)$ is \textbf{reverse} Lipschitz continuous with constant 
\begin{equation*}
    L_Q = \frac{L_r}{1-\gamma L_p(1+L_\pi)} .
\end{equation*}
\label{lem:q-rev-lipschitz}
\end{lemma}
Intuitively, this corresponds to an MDP where the reward function changes rapidly or the dynamics cause trajectories to diverge sufficiently quickly. With such a structure, the value function corresponding to any reverse Lipschitz (but otherwise arbitrary) policy also inherits this property.

Finally, to extend our results to the setting of limited accuracy, we formalize the idea of an $\varepsilon$-accurate value function with the usual definition below.
\begin{definition}[$\varepsilon$-Accurate Value]
An estimate for the action-value function, denoted $\widehat{Q}^\pi(s,a)$ is said to be $\varepsilon$-accurate if for some $\varepsilon>0$, the following is satisfied:
\begin{equation}
    \big|\widehat{Q}^\pi(s,a) - {Q}^\pi(s,a)\big| \leq \varepsilon,
\end{equation}
for all $s \in \s, a \in \A$.
\end{definition}
This error can arise from e.g. function approximation or early stopping, which we lump together for simplicity. In the following sections we discuss how such errors impact the problem of identifiability.

\section{Prior Work}
The simplest dynamics model available in most value-based RL algorithms is just the replay buffer: the dataset of stored transitions experienced during online interaction.
This primitive ``model'' acts as a lookup table where the ``closest'' match to a queried state-action tuple returns the correspondingly stored next-state. Naturally, more sophisticated (and useful) models can be derived. For example, leveraging the generalization power of neural networks one could train a deterministic model by regressing on the function $f$ discussed in Assumption~\ref{assump:det} during data collection. 

In recent years, advanced techniques specific to the RL problem have emerged for learning accurate world models~\cite{hafner2023mastering, schrittwieser2020mastering}. These methods often aim to generate optimal plans from ``imagined rollouts'' in latent space, which guide decision-making when online interaction resumes.

Our method, as outlined in the introduction, reuses previously trained $Q$ functions. Prior work has explored a variety of related strategies for effectively reusing such data. For instance~\cite{Todorov, Haarnoja2018, boolean, nemecek, Adamczyk_UAI} \textit{compose} the old value functions to obtain approximate solutions and bounds for new tasks. 
Separately,~\cite{taylor_survey, agarwal2022reincarnating, Adamczyk_AAAI, uchendu2023jump} have discussed workflows that enable agents to leverage old or suboptimal solutions, mitigating the need for retraining from scratch.

As discussed, our work diverges from these by using an inverse of the Bellman equation itself to derive a transition model directly from $Q$, rather than an online data stream. This innovation bridges ideas from both model-based RL and transfer learning, offering a unique perspective on data and model reuse. 
\section{Results}

We begin by noting the Bellman equation does not only hold for the optimal policy as shown in Equation~\eqref{eq:bellman}, but also holds for arbitrary policies:
\begin{equation}
    Q^\pi(s,a) = r(s,a) + \gamma \mathbb{E}_{s' \sim{} p(\cdot|s,a)} V^\pi(s'),
    \label{eq:bellman-pi}
\end{equation}
where $V^\pi$ is the (state) value function defined by the expectation over the policy: ${V^\pi(s)=\mathbb{E}_{a' \sim{} \pi(\cdot|s')} Q^\pi(s',a')}$. We write $Q^\pi$ and $V^\pi$ throughout to emphasize that a model can be obtained from the value of \textit{any policy} (not just the optimal policy) and under regularization (e.g. MaxEnt RL~\cite{haarnoja_SAC}).

A simple rearrangement of this equation ``solving'' for $s'$ gives, under deterministic dynamics 
\begin{align}
    V^\pi(f(s,a)) &= \frac{Q^\pi(s,a) - r(s,a)}{\gamma}.
    \label{eq:inverted-dynamics}
\end{align}

Indeed, under suitable assumptions, Eq.~\eqref{eq:inverted-dynamics} can be inverted to find the transition function:
\begin{proposition} Under Assumptions~\ref{assump:det} and~\ref{assump:rwd}, if an inverse state-value function exists, the successor state can be identified:
\begin{equation}
    s' = f(s,a) = \big[V^\pi\big]^{-1} \left( \frac{Q^\pi(s,a) - r(s,a)}{\gamma} \right).
    \label{eq:rev-bellman}
\end{equation}\label{prop:state-inverse}
\end{proposition}

This initial result shows that with the right assumptions, one can use an exact value function to calculate any successor state. However, obtaining an \textit{exact} $Q$-function is impractical, so we instead consider a value function suffering a globally bounded error: that is, we suppose an $\varepsilon$-accurate value function $\widehat{Q}^\pi$ is given. In the following two sections, we consider the case of continuous state spaces (where we use the reverse Lipschitz assumption) and the case of discrete state spaces (where we introduce a new definition necessary for identifiability).
\subsection{Theory for Continuous Spaces}

In the setting of an $\varepsilon$-accurate value, the function $\widehat{V}^\pi(s')$ cannot be inverted exactly, but the next-state can still be identified within an interval, leading to an extension of Proposition~\ref{prop:state-inverse}:

\begin{tcolorbox}[colback=blue!5!white,colframe=blue!75!black]
\begin{theorem}
Given an $\varepsilon$-accurate value function $\widehat{Q}^\pi(s,a)$ with reverse Lipschitz constant $L$, the error in estimating the next-state $s'$ from any $(s,a)$ is upper bounded:
\begin{equation}
    |s' - \widehat{s}'| < \frac{1+\gamma}{\gamma L} \varepsilon.
\end{equation}
\label{thm:sp-error}
\end{theorem}
\end{tcolorbox}

The proof and an intuitive visualization of Theorem~\ref{thm:sp-error} is provided in the Appendix. The error accumulates in both the queried value function $V^\pi(s')$ and also the ``scanned'' value, which we denote $\mathcal{V}$ (which depends on $(s,a)$ through the right-hand side of Eq.~\eqref{eq:inverted-dynamics}). With a lower bound on the value function's derivative, this region translates to a confidence interval over state space $\s$.

In practice, where obtaining exact value functions is infeasible, Theorem~\ref{thm:sp-error} demonstrates that even \textit{inexact} value functions can still be effective for deriving reliable models.

\subsection{Theory for Discrete Spaces}
In the case of discrete states, we need an alternative way to ensure the ``function'' (now a table) $V^\pi(s')$ remains invertible. A necessary condition for all states to be identifiable is that the corresponding values be distinct:
\begin{definition}[$\delta$-Separable Value Function]
A state value function $V$ is said to be $\delta$-\textit{separable} if there exists a $\delta>0$ such that
\begin{equation}
    | V(s) - V(x) | > \delta,
\end{equation}
for all $s,x \in \s$ such that $s \neq x$. 
\end{definition}

Finding sufficient structural assumptions for such separability seems to be a challenging problem in itself, which may be of independent interest. Nevertheless, if $\delta$-separability can be assumed (or determined \textit{a posteriori}), then identifiability holds. The following result ensures the next-state prediction problem remains identifiable, even in the case of errors:

\begin{tcolorbox}[colback=blue!5!white,colframe=blue!75!black]
\begin{theorem}[Successor-State Identifiability]
Suppose the true value function ($V^\pi$) is $\delta$-separable and an $\varepsilon$-accurate estimate of the value function ($\widehat{V}^\pi$) is given. If $\varepsilon< 
\delta (2\gamma^{-1}+2)^{-1}$, then the dynamics model is identifiable.\label{thm:tab-sp}
\end{theorem}
\end{tcolorbox}
Theorem~\ref{thm:tab-sp} is the analogue of Theorem~\ref{thm:sp-error} in the discrete case. Again, this result highlights that even inaccurate value functions can provide robust dynamics models which in this case return the exact successor state.
The proof of this result can be found in the Appendix.
Interestingly, in both continuous and discrete spaces, our analysis formally suggests that larger discount factors improve the model accuracy by (a) reducing the uncertainty in Theorem~\ref{thm:sp-error} and (b) increasing the minimum tolerance for identifiability in Theorem~\ref{thm:tab-sp}.

\section{Experiments}
As a proof of concept, we first consider a simple experiment in the tabular setting to verify our theoretical results. Here, the value function can be solved exactly with sufficiently many iterations of the Bellman optimality operator. In the following, we thus use the policy $\pi=\pi^*$, though the framework is agnostic to such a choice. Since the state space is discrete, we treat the $Q$ function with a $\max$ over action dimension as a lookup table, comparing its entries to the value of $\mathcal{V}$, calculated from the right-hand side of Eq.~\eqref{eq:rev-bellman}. We choose the index whose corresponding value is closest to $\mathcal{V}$, and the state index is considered the successor state prediction, $s'$. Indeed, when the value function is calculated to high precision, the accuracy of our method remains consistent: for all possible state-action pairs, the corresponding successor state is predicted successfully (corresponding to 100\% on the vertical axis of Fig.~\ref{fig:gap-vs-acc}).

To further test the accuracy of our model, in connection to the theoretical results derived, we compute the accuracy (again over all state-action pairs) for increasingly lower precision (that is, larger values of $\varepsilon$, in the definition of $\varepsilon$-accurate value function). To validate the idea of $\delta$-separability, we also prepare 20 reward-varying MDPs (with the same dynamics) each having a similar value gap, $\delta$ (within 1\% of the stated value). Due to space constraints, we give a full description of the experiment in the Appendix. We find these experiments support the result of Theorem~\ref{thm:tab-sp}.

\begin{figure}[H]
    \centering
    \includegraphics[width=0.4\textwidth]{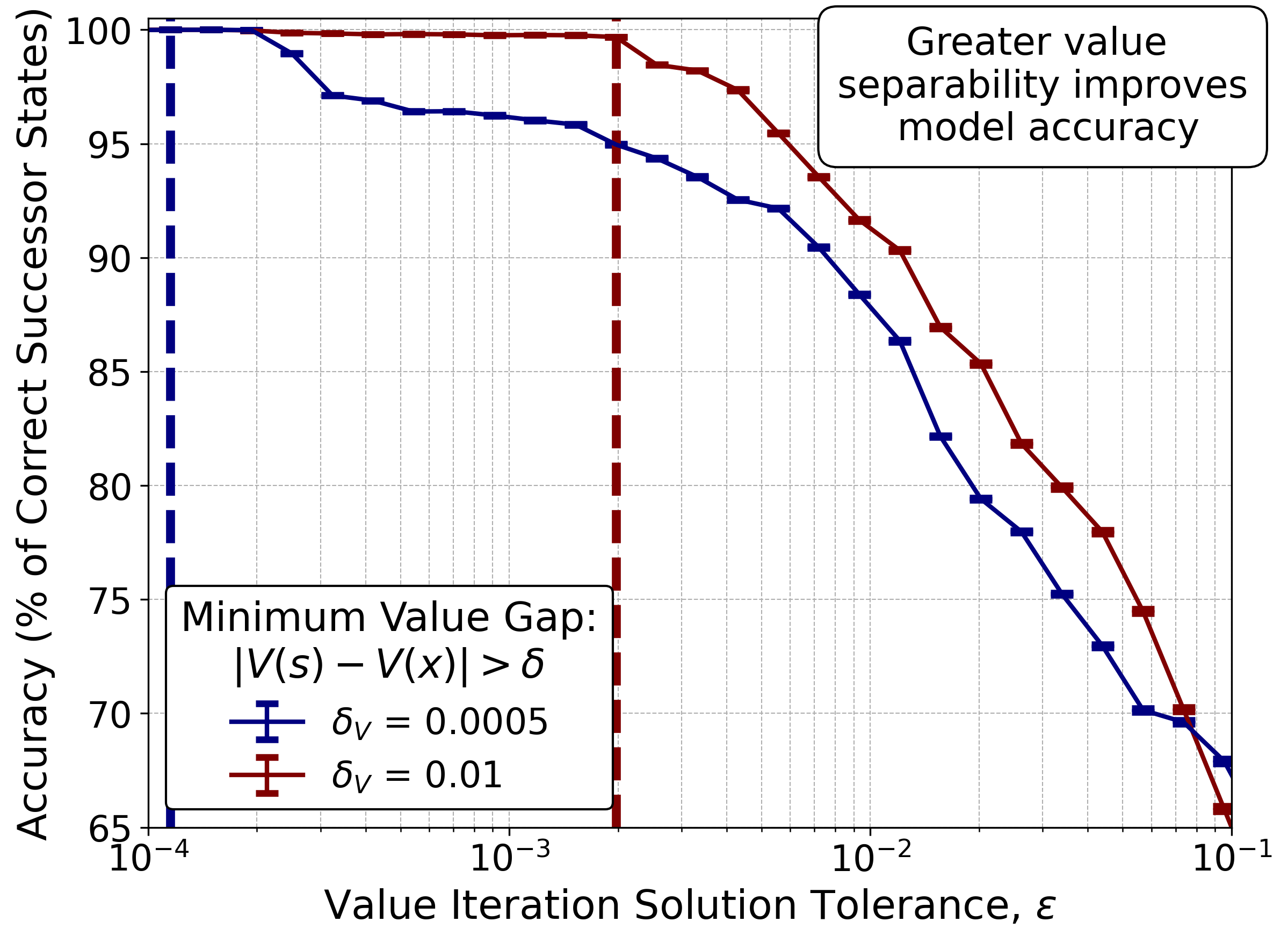}
    \caption{Value separability improves model accuracy. Bars denote the standard error in an average over 20 tasks solved to fixed tolerance (in some cases smaller than thickness of bar). The vertical dashed lines indicate the critical value given by the bound in Theorem~\ref{thm:tab-sp}.}
    \label{fig:gap-vs-acc}
\end{figure}

\section{Discussion}
The idea of $\delta$-separability seems closely related to the ``gap'' familiar from the bandit literature. It may be of interest in future work to extend the results here using analysis on a distribution of state value gaps, i.e. $\mathbb{P}\left(0 \leq \delta(s)\leq t\right) \leq c_g t^\zeta$ as in \cite{farahmand2011action}. This would allow for a probabilistic extension of Theorem~\ref{thm:tab-sp} in the continuous case, where one may employ a Gaussian prior over states as an inductive bias for locality.

In general, a single non-invertible value function poses issues for identifiability. Explicitly requiring trained (or distilled) value functions to be invertible~\cite{ardizzone2018analyzing} may be a useful path forward. However, solutions may still be recovered in the non-invertible regime. For example, one may solve for the \textit{level set}, ${\overline{S}=\{s\in \s : V^\pi(s)=\mathcal{V}\}}$, which now represents the set of all possible successor states consistent with the expected state-value function. The degeneracy $|\overline{S}|>1$ can potentially be bypassed with additional structural assumptions. For example, with sufficiently ``smooth'' dynamics, the distance between any two consecutive states can be bounded. Thus, the cardinality of the level set can be reduced by only considering states $s' \in \overline{S}$ within some neighborhood of the current state. 

An independent technique for ``pruning'' the level set is to use \textit{multiple} cached value functions: for example the values of different policies or the optimal value functions for different (e.g. reward-varying, discount-varying) tasks. These value functions provide unique level sets which must each contain the true successor state. Thus, $s'$ can be found by continually pruning the level set: $s' \in \bigcap_{i=1}^{k} \overline{S}_i$.

With a sufficient number of diverse value functions (i.e. diverse enough to generate non-identical level sets), the successor state can be systematically recovered. 

We find the proposed approach to dynamics modeling both simple and effective. We believe this new perspective has the potential to generate new algorithms at the intersection of model-based and model-free RL while simultaneously opening avenues for deeper theoretical investigation. 

\section{Acknowledgements}
JA would like to acknowledge helpful discussions with Rahul V. Kulkarni and Josiah C. Kratz. JA acknowledges funding support from the NSF through Award No. DMS-1854350 and PHY-2425180. This work is supported by the National Science Foundation under Cooperative Agreement PHY-2019786 (The NSF AI Institute for Artificial Intelligence and Fundamental Interactions, http://iaifi.org/).
\bibliography{aaai25}
\clearpage
\onecolumn
\appendix
\section{Appendix: Proofs}
\subsection{Proof of Theorem~\ref{thm:sp-error}}
We will first restate the theorem for convenience:
\begin{theorem*}
Given an $\varepsilon$-accurate value function $\widehat{Q}^\pi(s,a)$ with reverse Lipschitz constant $L$, the error in estimating the next-state $s'$ from any $(s,a)$ is upper bounded:
\begin{equation*}
    |s' - \widehat{s}'| < \frac{1+\gamma}{\gamma L} \varepsilon.
\end{equation*}
\end{theorem*}

\begin{figure}[H]
    \centering
    \includegraphics[width=0.4\textwidth]{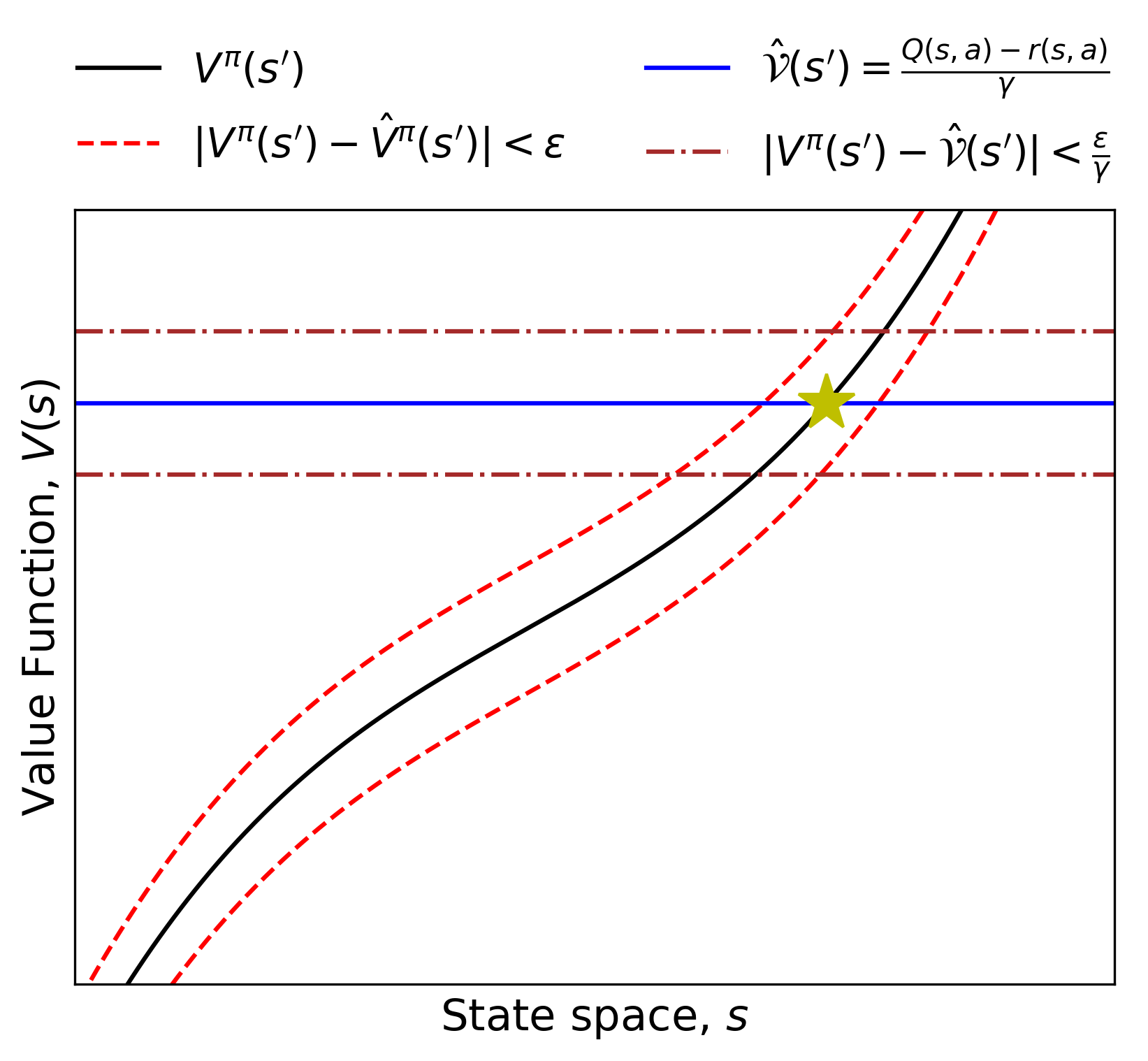}
    \caption{Illustration of the errors described in Theorem~\ref{thm:sp-error}. The dashed lines represent uncertainty bounds on $V^\pi$ and the scanned value $\mathcal{V}$ from the Bellman equation, forming an uncertain region around the star (the true next-state and corresponding value).  Note: although we plot the value as a function of state $s$, in practice this should be interpreted as the space of potential successor states, $s'$. }
    \label{fig:my_label}
\end{figure}

\begin{proof}
First note that if the action-value error is bounded, the state-value error is equally bounded, as shown below:
\begin{align*}
    |\widehat{V}^\pi(s) - V^\pi(s)| &= |\mathbb{E}_{a \sim{} \pi(\cdot|s)} \widehat{Q}^\pi(s,a) - \mathbb{E}_{a \sim{} \pi(\cdot|s)} {Q}^\pi(s,a)| \\
    &\leq \mathbb{E}_{a \sim{} \pi(\cdot|s)} |\widehat{Q}^\pi(s,a) - Q^\pi(s,a)| \\
    &\leq \mathbb{E}_{a \sim{} \pi(\cdot|s)} \varepsilon = \varepsilon,
\end{align*}
where we use the triangle inequality in the penultimate line. Additionally, note that the ``scanned value'' $\mathcal{V}=\gamma^{-1}(Q-r)$ also suffers a bounded error. Assuming the reward function and discount factor are known exactly:
\begin{align*}
    \big|\widehat{\mathcal{V}}-\mathcal{V}\big| &= \biggr| \frac{\widehat{Q}^\pi(s,a) - r(s,a)}{\gamma} - \frac{Q^\pi(s,a) - r(s,a)}{\gamma} \biggr|\\
    &= \gamma^{-1} | \widehat{Q}^\pi(s,a) - Q^\pi(s,a)|\\
    &\leq \gamma^{-1} \varepsilon.
\end{align*}
Notice that in the previous bound one can readily extend the analysis for $\varepsilon$-accurate reward functions and discount factors.

With these terms bounded, we can now suppose that the value function changes at its smallest rate: $L$, which incurs the worst error. The function $V^\pi$ and the inverse value $\mathcal{V}$ both have bounded error, which in turn implies that the error in estimating the inverse (the state) is also bounded.
From the first ``function evaluation'' error, let $\delta s_1$ denote the error in $s$ caused by incorrectly using the value of $\hat{V}(s)$, depicted by the intersection of the blue line and red dashed lines in Fig.~\ref{fig:my_label}. Then, the resulting error can be bounded as
\begin{align*}
    \frac{\varepsilon}{\delta s_1} = V'(s) &> L, \\
    \delta s_1 < L^{-1} \varepsilon,
\end{align*}
and similarly for the second ``inverse value'' error:
\begin{align*}
    \frac{\varepsilon/\gamma}{\delta s_2} = V'(s) &> L, \\
    \delta s_1 < L^{-1} \varepsilon / \gamma,
\end{align*}
where $V'$ denotes the derivative of the state-value function with respect to $s$.
To complete the proof, we note that the worst-case error occurs when these terms combine additively, resulting in the stated error bound: 
\begin{align*}
    |s' - \widehat{s}'| &< \delta s_1 + \delta s_2 \\
    &< L^{-1} \varepsilon + L^{-1} \varepsilon / \gamma \\
    &= \frac{1+\gamma}{\gamma L} \varepsilon.
\end{align*}
\end{proof}
\subsection{Proof of Theorem~\ref{thm:tab-sp}}
\begin{theorem*}[Successor-State Identifiability]
Suppose the true value function ($V^\pi$) is $\delta$-separable and an $\varepsilon$-accurate estimate of the value function ($\widehat{V}^\pi$) is given. If $\varepsilon< 
\delta (2\gamma^{-1}+2)^{-1}$, then the dynamics model is identifiable.
\end{theorem*}

\begin{proof}
In the tabular setting, it is clear that $\delta$-separability for any $\delta > 0$ allows for states to be distinguished, since each state has a unique corresponding value. However, incorporating an error of $\varepsilon$ in the value function has a non-obvious effect on the ability to identify states. Specifically, if $\varepsilon$ is too large, the margin of error for two states may overlap, i.e., $V$ may no longer be separable for any $\delta > 0$. Thus, we must first show that the $\varepsilon$-accurate model $\widehat{V}^\pi$ is also $\delta$-separable.

The separability condition is determined by the smallest gap between any two state values. Without loss of generality, assume $V^\pi(s) > V^\pi(x)$ represent the two closest state values. Then, by definition of separability, the states $s$ and $x$ must satisfy the defining bound:
\begin{equation*}
    V^\pi(s) - V^\pi(x) > \delta.
\end{equation*}

In the worst-case scenario, the estimate of the highest value $\widehat{V}^\pi(s)$ is an overestimate of $V(s)$ by $\varepsilon$, and the estimate of the next largest value $\widehat{V}^\pi(x)$ ``adversarially'' underestimates the true value $V^\pi(x)$ by $\varepsilon$. This reduces the gap between the values as follows:
\begin{align*}
    \widehat{V}^\pi(s) - \widehat{V}^\pi(x) &> {V}(s) - \varepsilon - \left({V}^\pi(x) + \varepsilon\right)\\
    &= \delta - 2\varepsilon,
\end{align*}
where the first line follows from the definition of $\varepsilon$-accuracy, and the second line follows from the $\delta$-separability of the exact value function, $V^\pi(s)$. Since this setting represents the closest any two states can be in value under $\widehat{V}^\pi(s)$, we have for all states $s,x \in \s$:
\begin{equation*}
|\widehat{V}^\pi(s) - \widehat{V}^\pi(x)| > \delta - 2\varepsilon.
\end{equation*}

As expected, using an imprecise value function reduces the minimum possible gap.

Now that the estimated value function is $\delta$-separable, the proposed algorithm would involve ``scanning'' for the value $\mathcal{V}(s')$ among the table's entries. However, the scanned value is also subject to error. In the proof of Theorem~\ref{thm:sp-error} we showed that $\big|\widehat{\mathcal{V}}-\mathcal{V}\big|<\gamma^{-1}\varepsilon$, which equally holds in the present context.

For the next-state to be properly identified, the margin of error implied by $\widehat{\mathcal{V}}$ must not overlap with the remaining separability gap, of size $\delta - 2\varepsilon$. To be precise: given $\widehat{\mathcal{V}}$ the entire range of plausible underlying $\mathcal{V}$ values lie in a range of size $2\gamma^{-1} \varepsilon$. Thus we require 
\begin{equation*}
    \delta - 2 \varepsilon > 2 \gamma^{-1} \varepsilon,
\end{equation*}
for there to remain a gap between state values. This imposes the constraint on $\varepsilon$:
\begin{equation*}
    \varepsilon < \frac{\delta}{2\gamma^{-1}+2},
\end{equation*}
which is the bound shown in the main text.
\end{proof}

\section{Appendix: Experiments}
We describe the details of the experiments shown in Fig.~\ref{fig:gap-vs-acc} below.  First, the deterministic dynamics are fixed for an ``empty'' $5 \times 5$ gridworld (meaning: no obstacles, 4 actions in cardinal directions). We then generate reward functions uniformly at random within the range $r(s) \in (-1,1)^{|S|}$. Each reward function defines a task, which is solved to determine the value of $\delta$, corresponding to the minimum gap in state values after optimization. Once a reward function has been determined to have a value of $\delta$ within 1\% of the desired value (as presented in the legend), we solve this task over a range of precision values ($\varepsilon$, on $x$-axis in Fig.~\ref{fig:gap-vs-acc}). The value of $\varepsilon$ is controlled by using an appropriate number of iterations of the Bellman operator. The discount factor is fixed to $\gamma=0.99$.

\end{document}